\documentclass{article}

\usepackage{graphicx}
\usepackage{float}
\usepackage[hypcap=false]{caption}
\usepackage{subcaption}
\usepackage{booktabs} 

\usepackage{hyperref}

\usepackage{microtype}

\usepackage[preprint]{icml2026}

\usepackage{amsmath}
\usepackage{amssymb}
\usepackage{mathtools}
\usepackage{amsthm}
\usepackage{enumitem}
\usepackage[table]{xcolor}
\definecolor{lightgray}{gray}{0.95}

\usepackage[capitalize,noabbrev]{cleveref}

\theoremstyle{plain}
\newtheorem{theorem}{Theorem}[section]
\newtheorem{proposition}[theorem]{Proposition}
\newtheorem{lemma}[theorem]{Lemma}
\newtheorem{corollary}[theorem]{Corollary}
\theoremstyle{definition}
\newtheorem{definition}[theorem]{Definition}
\newtheorem{assumption}[theorem]{Assumption}
\theoremstyle{remark}
\newtheorem{remark}[theorem]{Remark}

\newcommand{\numseeds}{three}
\newcommand{\satwikaffil}{Independent Researcher}

\icmltitlerunning{Recoverability Has a Law}

\begin{document}

\twocolumn[
  \icmltitle{%
    \texorpdfstring
        {Recoverability Has a Law: \\ The ERR Measure for Tool-Augmented Agents}
        {Recoverability Has a Law: The ERR Measure for Tool-Augmented Agents}
  }

  \icmlsetsymbol{equal}{*}

  \begin{icmlauthorlist}
    \icmlauthor{Sri Vatsa Vuddanti}{equal,sjsu}
    \icmlauthor{Satwik Kumar Chittiprolu}{equal,comp}
  \end{icmlauthorlist}

  \icmlaffiliation{sjsu}{San Jose State University, San Jose, CA, USA}
  \icmlaffiliation{comp}{\satwikaffil}

  \icmlcorrespondingauthor{Sri Vatsa Vuddanti}{srivatsa644@gmail.com}
  \icmlcorrespondingauthor{Satwik Kumar Chittiprolu}{satchi427@gmail.com}

  \icmlkeywords{Machine Learning, ICML}

  \vskip 0.3in
]

\printAffiliationsAndNotice{\icmlEqualContribution}

\begin{abstract}
Language model agents often appear capable of self-recovery after failing tool call executions, yet this behavior lacks a formal explanation. We present a predictive theory that resolves this gap by showing that recoverability follows a measurable law. To elaborate, we formalize recoverability through Expected Recovery Regret (ERR), which quantifies the deviation of a recovery policy from the optimal one under stochastic execution noise, and derive a first-order relationship between ERR and an empirical observable quantity, the \textit{Efficiency Score (ES)}. This yields a falsifiable first-order quantitative law of recovery dynamics in tool-using agents. We empirically validate the law across four tool-use benchmarks spanning controlled perturbations, diagnostic reasoning, and real-world APIs, with an additional diagnostic evaluation on three live production APIs. Across model scales, perturbation regimes, and recovery horizons, predicted regret under the ERR–ES law closely matched observed post-failure regret measured from Monte Carlo rollouts, within $\Delta_\text{norm} \leq 0.05$. Our results reveal that recoverability is not an artifact of model scale or architecture, but a governed property of interaction dynamics, providing a theoretical foundation for execution-level robustness in language agents.
\end{abstract}

\section{Introduction}
\label{sec:intro}

Tool-augmented language models increasingly operate as autonomous controllers over APIs, databases, and multi-step workflows \citep{qin2023toolllmfacilitatinglargelanguage, li2023apibankcomprehensivebenchmarktoolaugmented}.

In these interactive settings, robustness is no longer a property of the \emph{input} alone, but of the \emph{execution process} itself. A single malformed API response, transient failure, or schema

\begin{minipage}{\columnwidth}
\centering
\includegraphics[width=0.95\columnwidth]{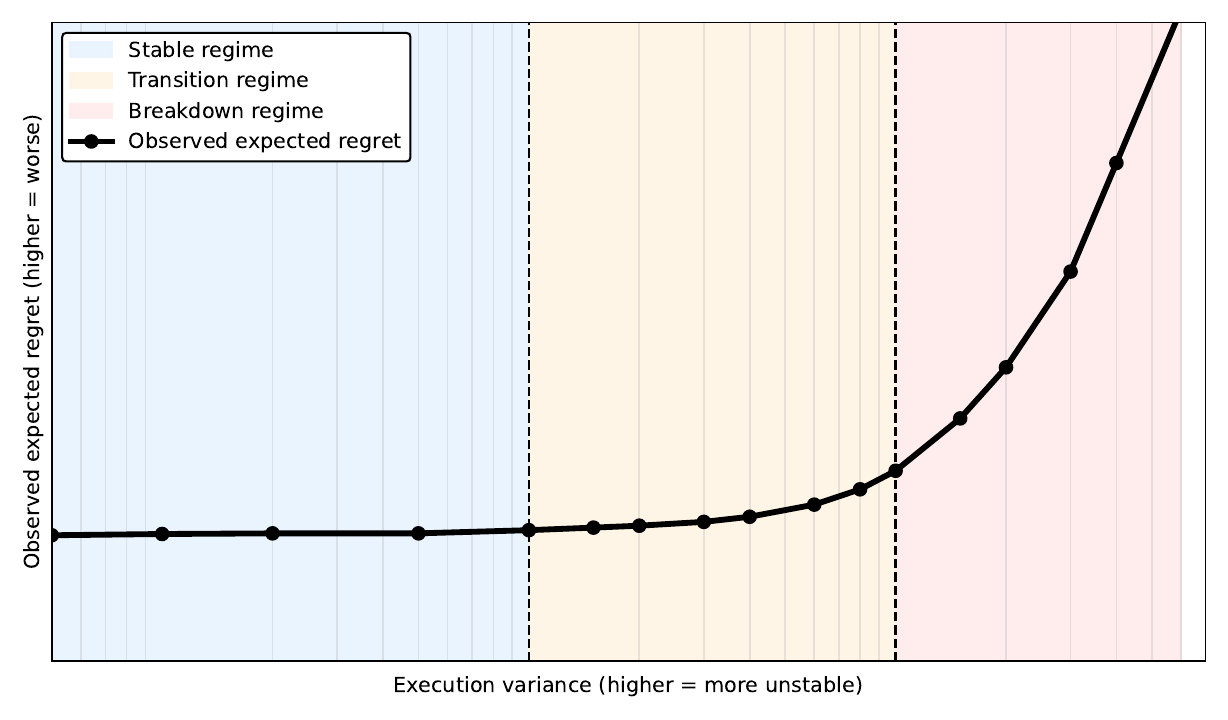}
\captionof{figure}{Recoverability follows a measurable law under execution noise. Observed recovery regret remains predictable in a low-variance regime and breaks sharply as execution instability increases. This phase transition motivates the ERR formulation and its efficiency-based predictor.}
\label{fig:law_overview}
\end{minipage}

drift can cascade through a chain of actions, producing errors that unfold across time rather than within a fixed feature space.

Despite the inherently temporal structure of tool interaction, dominant robustness paradigms—distributional and adversarial—treat perturbations as static, pre-input events \citep{madry2019deeplearningmodelsresistant, zhang2019theoretically, goodfellow2015explaining}. As a result, the dynamics of execution and recovery remain analytically unmodeled.

This gap becomes critical as language model agents are increasingly deployed as independent execution units in high-stakes settings, where failures propagate across time rather than across inputs. In practice, such agents often recover from execution errors, yet current theory offers no predictive account of when recovery succeeds, when it fails, or how errors compound over long horizons \citep{shinn2023reflexion, gou2024criticlargelanguagemodels, vuddanti2025paladinselfcorrectinglanguagemodel}.

We formalize recoverability via the \emph{Expected Recovery Regret (ERR)}, which quantifies the performance gap between a stochastic recovery policy and an optimal one under execution noise. Our analysis derives a closed-form, first-order coupling between ERR and an empirically observable quantity, the \emph{Efficiency Score (ES)}. This yields a falsifiable quantitative law governing recovery dynamics, with distributional and execution-level robustness emerging as limiting cases under different assumptions on temporal coupling and perturbation structure.

We use the term law in the same sense as scaling laws: an empirically stable regularity with a defined regime of validity, not a universal invariant.

We test the law using controlled recovery policies that span the ERR–ES manifold. Across controlled benchmarks, diagnostic reasoning tasks, real-world APIs, and open-domain tool use, predicted regret under the ERR--ES law closely matches observed post-recovery outcomes. This agreement holds across model scales (8B–32B), perturbation regimes, and long recovery horizons.

Together, these results show that robustness in tool-augmented language models is governed not by static invariance, but by measurable recoverability. The recovery law would hold even if our tested models were replaced by any policy that reduces variance and bounded cost under execution noise. By establishing and validating the ERR--ES law, this work reframes robustness as a quantifiable, law-governed property of interaction dynamics, rather than an emergent byproduct of scale or architecture.

\section{Preliminaries}
\label{sec:preliminaries}
We formalize the execution-level setting, assumptions, and notation used throughout the paper. These definitions specify the regime in which the recovery law is derived and evaluated.

\subsection{Execution-Level Setting}
We consider interactive agents that issue tool calls over multiple timesteps. At each step $t$, the system observes a state $s_t$, executes an action $a_t$, and receives a potentially perturbed outcome governed by a stochastic process $\mathcal{F}$. Perturbations may arise from latency, schema drift, or partial tool failures and affect state transitions rather than the static input. We refer to this regime as \textbf{execution-level robustness}.

\subsection{Assumptions}

We model the agent as a policy $\pi_\theta(a_t \mid s_t)$ interacting with a finite-horizon environment under stochastic execution noise $\delta_t \sim \mathcal{F}(s_t,a_t)$. The analysis assumes:

\begin{enumerate}[leftmargin=*]
\item \textbf{Bounded cost:} per-step execution cost satisfies $c_t \in [0, c_{\max}]$.
\item \textbf{Regular loss:} the loss $\ell(s_t,a_t,\delta_t)$ is Lipschitz-continuous in $\delta_t$.
\item \textbf{Stationary perturbations:} $\mathcal{F}(s_t,a_t)$ is time-invariant.
\item \textbf{Discounted horizon:} $\gamma \in (0,1)$ with $\gamma = 0.9$ in all experiments.
\end{enumerate}

These conditions place execution noise in a linear-response regime, enabling a first-order approximation of excess recovery loss.

\subsection{Analytical Objective Preview}
Under these assumptions, recoverability can be expressed as the deviation between the expected loss of a given policy and that of an optimal recovery policy acting under the same stochastic process. The next section formalizes the execution-level setting, defines the expected recovery regret (ERR), and introduces the empirical efficiency metrics used to test the recovery law.

\section{Formal Problem Definition and Metrics}
\label{sec:problem}

\subsection{Execution-Level Robustness}

We consider an interactive agent with policy $\pi_\theta(a_t \mid s_t)$ acting over states $s_t \in \mathcal{S}$ and actions $a_t \in \mathcal{A}$, where each $a_t$ issues a tool call or structured reasoning step. At each timestep, stochastic execution noise $\delta_t \sim \mathcal{F}(s_t,a_t)$ captures schema drift, latency, malformed outputs, or other runtime perturbations. State transitions follow
\[
s_{t+1} = T(s_t,a_t,\delta_t),
\]
and each step incurs a bounded, Lipschitz loss $\ell(s_t,a_t,\delta_t)$, ensuring stable deviation under small perturbations.

The total discounted loss of a trajectory $\tau$ is
\begin{equation}
L(\pi_\theta,\mathcal{F})
=
\mathbb{E}_\tau\!\left[
\sum_{t=0}^{H} \gamma^t \ell(s_t,a_t,\delta_t)
\right],
\qquad \gamma \in (0,1).
\end{equation}
Execution-level robustness thus concerns the agent's ability to correct within a trajectory rather than resist static input perturbations.

\subsection{Expected Recovery Regret (ERR)}

Let $\pi^*$ denote the optimal recovery policy under perturbation process $\mathcal{F}$:
\[
\pi^* = \arg\min_\pi L(\pi,\mathcal{F}).
\]
The \textbf{Expected Recovery Regret (ERR)} of $\pi_\theta$ is
\begin{equation}
\label{eq:err}
\text{ERR}(\pi_\theta)
=
\mathbb{E}_{\mathcal{F}}\!\left[
L(\pi_\theta,\mathcal{F}) - L(\pi^*,\mathcal{F})
\right].
\end{equation}
ERR measures the excess loss incurred by a recovery policy relative to the optimal one. Direct evaluation is infeasible because neither $\pi^*$ nor the full perturbation process is known.

\paragraph{Why ERR is a law of expectations.}
Execution noise is stochastic and trajectory-dependent; therefore any recovery law must govern expected regret rather than individual rollouts. Instance-level outcomes are intentionally not predicted, just as classical regret bounds do not predict single trajectories. We approximate $\pi^*$ as the best achievable recovery behavior under identical tool access and execution constraints, estimated empirically across seeds. The ERR--ES law is therefore a structural law over distributions of executions, not pointwise behavior.

\subsection{Observable Surrogates for Recovery Regret}
\label{sec:metrics}

We estimate recoverability from observable rollouts.

\paragraph{Recovery Rate (RR).}
The success probability of a full trajectory:
\[
\text{RR}
=
\mathbb{E}[\mathbb{I}(\text{success})].
\]
RR serves as a coarse measure of per-trajectory correction.

\paragraph{Cost-Sensitive Recovery (CSR).}
For completeness, we report $\text{CSR} = \text{RR} - \lambda C_r$, where $C_r \in [0,1]$ is the fraction of trajectory steps spent on corrective (re-issued) tool calls. CSR penalizes recovery \emph{effort}, measured in corrective steps; ES (below) penalizes recovery \emph{token cost}. The two cost measures are distinct and are not interchangeable. CSR is not used in the theoretical analysis but provides an interpretable accuracy–effort tradeoff.

\paragraph{Efficiency Score (ES).}
The ERR bound depends only on two observable aggregates: success and normalized cost.
\begin{equation}
\label{eq:es}
\text{ES}
=
\frac{\text{RR}}{1 + \lambda\, C},
\end{equation}
where $C \in [0,1]$ is cumulative trajectory cost normalized by the global maximum observed trajectory cost $C_{\max}$, and $\lambda > 0$ is a fixed cost weight ($\lambda = 0.5$ in all experiments). Normalizing $C$ ensures comparability across environments with heterogeneous tool costs.

Under bounded cost and Lipschitz losses, ERR admits the upper bound
\begin{equation}
\label{eq:err_es_bound_main}
\text{ERR}(\pi)
\;\le\;
\frac{1}{1-\gamma} (1 - \text{ES})
+
O(\lambda c_{\max}),
\end{equation}
derived from a first-order excess-loss decomposition standard in robust MDPs and risk-sensitive control (Appendix~\ref{appendix:proofs}). This approximation holds when execution variance is small and cost-weighted deviations grow approximately linearly with perturbation magnitude.

We now show that ES is not a design choice, but the only surrogate compatible with linearized recovery regret. We first fix the class of surrogates under consideration.

\begin{definition}[Admissible Surrogate Class]
\label{def:admissible}
A scalar surrogate $S:[0,1]^2 \to [0,1]$, $S = S(\mathrm{RR}, C)$, is \emph{admissible} if
(a) it depends on rollouts only through the trajectory-level aggregates $\mathrm{RR}$ and $C$;
(b) it is non-decreasing in $\mathrm{RR}$ and non-increasing in $C$;
(c) $S(1, 0) = 1$ and $S(0, C) = 0$ for all $C$; and
(d) its complement $1 - S$ is affine-equivalent to $\mathrm{ERR}$ under the first-order linearization of Eq.~\ref{eq:err_es_bound_main}, i.e.\ $1 - S = a\cdot\mathrm{ERR} + b$ for some $a > 0$, $b \in \mathbb{R}$ up to $O(\lambda c_{\max})$.
\end{definition}

\begin{theorem}[Uniqueness of the Efficiency Surrogate]
\label{prop:es_necessity}
Assume excess recovery loss admits a first-order linearization in cost-weighted failure (Lemma~\ref{lem:linear_loss_gap}). Then any admissible surrogate in the sense of Definition~\ref{def:admissible} must take the form
\[
S = \frac{\mathrm{RR}}{1 + \lambda C},
\]
up to positive affine rescaling. In particular, ES is not optimized by any method in our experiments; it emerges as a coordinate that organizes recovery behavior even when systems are not trained to maximize it.
\end{theorem}

A proof sketch is given in Appendix~\ref{appendix:es_uniqueness}.

\begin{corollary}
\label{cor:additive_fails}
Any surrogate additive in RR and $C$ (e.g.\ $\mathrm{RR} - \lambda C$, the token-cost analogue of CSR), or separable across steps, is not affine-equivalent to ERR under bounded cost: it may preserve the ordering of policies but cannot be calibrated to predict the magnitude of regret.
\end{corollary}

Appendix~\ref{appendix:surrogates} tests this directly: substituting an additive or harmonic-mean surrogate for ES in the law preserves the ranking of recovery policies but inflates prediction error from under $5\%$ to $12$--$31\%$.

\paragraph{Empirical correction.}
Real tool-use environments exhibit mild heterogeneity in trajectory costs. We therefore track an empirical predictor
\[
\widehat{\text{ERR}}_{\text{law}}
=
\frac{1}{1-\gamma}(1-\text{ES}).
\]
All empirical law validations in the main paper use this uncorrected first-order predictor; no fitted or calibrated terms are used to evaluate adherence to the ERR--ES law.

\paragraph{Theoretical role of $\beta$.}
Under heterogeneous trajectory costs, the first-order excess-loss decomposition introduces a bias term proportional to the variance of $C$. The coefficient $\beta$ corresponds to this analytic curvature term (Appendix~\ref{appendix:variance_bias}), and its empirical estimation calibrates only the magnitude, not the structure, of the correction. Thus $\beta$ quantifies second-order curvature neglected by the first-order law.

\subsection{Analytic Properties of the Recovery Law}

We highlight two consequences of the ERR--ES coupling.

\paragraph{Behavioral limits.}
When $\lambda \to 0$, ES reduces to RR; when $\gamma \to 1$, ERR approaches the infinite-horizon regret of risk-sensitive RL.

\begin{corollary}[Linear Regime]
\label{cor:tightness}
Under constant per-step cost $c_t=c\le c_{\max}$, the bound in Eq.~\ref{eq:err_es_bound_main} is tight up to $O(\lambda c_{\max})$:
\[
\text{ERR}(\pi)
=
\frac{1}{1-\gamma}(1-\text{ES}) + O(\lambda c_{\max}).
\]
\end{corollary}

This regime illustrates the limiting case in which the ERR--ES law is exactly linear.

\paragraph{Scope of the Recovery Law.}
ERR--ES provides a first-order characterization of recoverability. It captures the dominant contribution to recovery regret under bounded cost, stationary perturbations, and low execution variance. Higher-order effects appear only as controlled residuals (Appendix~\ref{appendix:variance_bias}) and vanish as recovery stabilizes.

The intent is not exactness, but falsifiability: a minimal law that governs recovery dynamics in practical tool-use regimes.

\begin{corollary}[Early-Step Stability]
\label{cor:early_stability}
Under discounting $\gamma<1$, higher early-step efficiency reduces downstream regret multiplicatively. In particular, recovery policies with higher ES at early timesteps suppress compounding error cascades.
\end{corollary}

\section{Empirical Framework for Law Validation}
\label{sec:method}

This section describes how the theoretical quantities introduced in Section~\ref{sec:problem}—the \emph{Expected Recovery Regret (ERR)} and the \emph{Efficiency Score (ES)}—are instantiated to enable empirical validation of the recovery law. We define a minimal evaluation framework that operationalizes the law's assumptions and permits controlled measurement in tool-use environments.

Throughout, we use \textbf{FORTIFY} as a concrete instantiation of this framework: a lightweight recovery mechanism that realizes the ERR objective under bounded cost and stationary perturbations. The recovery law itself is agnostic to this choice; FORTIFY serves only as a tractable evaluation vehicle. We make no claim that FORTIFY is a practical recovery system, and its absolute performance in Section~\ref{sec:results} should be read as one point on the efficiency--regret manifold rather than as a method contribution.

\subsection{Principles for Empirical Approximation of ERR}

Because $\pi^*$ is unobservable, we evaluate the law using a family of recovery policies that generate diverse (RR, C) operating points. These policies are instruments for probing the ERR–ES manifold, not objects of optimization.

\subsection{Approximation of the Recovery Policy}

To approximate $\pi^*$ for evaluation, we employ a supervised recovery prior trained on synthetic failure trajectories $(s_t,a_t,\delta_t)$ generated by a stochastic failure simulator. Each trajectory pairs a perturbed tool interaction with a corrected counterpart, yielding structured recovery examples.

The resulting policy is not assumed optimal; it provides a stable reference whose behavior can be analyzed under the ERR--ES framework.

\subsection{Variance Reduction via Retrieval Conditioning}

At inference time, the recovery prior may be augmented with retrieval-based conditioning \citep{borgeaud2022improving, wu2022memorizingtransformers}. Given a state $s_t$, the policy retrieves $K$ semantically similar failure–recovery exemplars and conditions its next action on them:
\[
\pi_{\theta}^{\text{RTV}}(a_t \mid s_t)
   = \pi_{\theta}(a_t \mid s_t, r_{1:K}).
\]

This conditioning aligns local recovery decisions with previously observed failure structures, empirically reducing the variance component of ERR (Equation~\ref{eq:err_es_bound_main}) without modifying model parameters.

\subsection{Operational Measurement and Falsifiability}

During interaction, the framework executes an iterative recovery loop that detects failures, conditions on prior recovery structure when available, and reissues tool calls until success or cost-budget exhaustion. From these trajectories, we compute recovery rate (RR) and efficiency score (ES), which serve as empirical counterparts of ERR.

Because the ERR--ES law makes quantitative predictions about how regret scales with efficiency, any recovery policy instantiated within this framework constitutes a test point rather than a privileged solution.

\section{Experimental Setup}
\label{sec:exp}

We evaluate whether the ERR--ES law predicts empirical recoverability across models, tools, and perturbation regimes. The goal is not to compare recovery systems per se, but to test the law's \emph{predictive validity}, \emph{generality}, and \emph{robustness} under increasingly stochastic execution noise.

\subsection{Evaluation Questions}

The experiments are designed to test three hypotheses implied directly by the ERR--ES relationship:
\begin{enumerate}
\item \textbf{Predictive validity:}
empirical recovery regret scales with $(1-\text{ES})/(1-\gamma)$, as predicted by the law.
\item \textbf{Variance sensitivity:}
reductions in trajectory-level variance tighten adherence to the predicted efficiency--regret coupling.
\item \textbf{Generality:}
the same coupling holds across synthetic, diagnostic, and open-world perturbation regimes without environment-specific tuning.
\end{enumerate}

\subsection{Benchmarks and Perturbation Regimes}
\label{sec:benchmarks}

We evaluate recoverability across four benchmarks that span increasing degrees of stochasticity:
\textsc{PaladinEval} (controlled perturbations),
\textsc{ToolReflectEval} (diagnostic reasoning failures),
\textsc{API--Bank} (real APIs with schema drift), and
\textsc{WikidataLive} (cross-domain semantic queries).
As an additional diagnostic setting, we evaluate on three live production APIs (finance, travel, weather) exhibiting non-stationary execution noise; because these traces are fewer and non-stationary, we report them qualitatively in Appendix~\ref{appendix:live_api} rather than in the main tables. Together, these environments enable testing whether the ERR--ES coupling persists when the underlying perturbation process $\mathcal{F}$ is partially or fully unknown. Additional task details appear in Appendix~\ref{sec:benchmark_depth}.

\subsection{Models and Recovery Policies}

We evaluate four open-weight backbones: Qwen~2.5 (14B), Gemma~3 (12B), Thinking~V1 (32B), and Llama~3 (8B). To generate diverse recovery behaviors, we instantiate several commonly used recovery policies, including supervised correction, self-reflection, critique-based revision, and a non-recovery baseline. All methods share identical prompting formats, tool schemas, and execution budgets to isolate differences in recovery dynamics. Implementation details are provided in Appendix~\ref{sec:implementation}.

\subsection{Protocol and Metrics}
\label{sec:protocol}

For each benchmark, model, and recovery policy, we perform $200$ Monte-Carlo rollouts with perturbations $\delta_t \sim \mathcal{F}(s_t,a_t)$. We compute recovery rate (RR) and efficiency score (ES) as:
\[
\text{RR}
=
\frac{1}{N}\sum_i \mathbb{I}_i,
\qquad
\text{ES}
=
\frac{1}{N}\sum_i
\frac{\mathbb{I}_i}{1+\lambda C_i},
\]
where $C_i$ denotes normalized cumulative cost. We fix $\gamma = 0.9$ and $\lambda = 0.5$ across all benchmarks, consistent with short-horizon interactive settings. All results report means and $95\%$ bootstrap confidence intervals \citep{efron1994introduction} over \numseeds{} random seeds. Unless otherwise stated, all reported predicted ERR values use the first-order law with $\beta = 0$.

\subsection{Operational Definitions}
\label{sec:operational}

The quantities in Section~\ref{sec:problem} are computed from raw tool traces as follows.

\paragraph{Per-step loss.}
For every benchmark, the per-step loss is
\begin{equation}
\label{eq:loss_instantiation}
\ell(s_t, a_t, \delta_t) = \mathbb{I}[\text{fail}_t] + \lambda\,\frac{c_t}{c_{\max}},
\end{equation}
where $\mathbb{I}[\text{fail}_t]$ indicates that step $t$ ends in an unresolved tool failure, $c_t$ is the token cost (prompt plus completion tokens) of step $t$, and $c_{\max}$ is the global maximum per-step cost observed across all benchmarks. A step ends in an unresolved failure if the tool call returned an error, a schema-invalid response, or a semantically incorrect result (as judged by the benchmark's reference checker) and the agent did not issue a corrective call that resolved it within the same step. The first term makes $L(\pi,\mathcal{F})$ a discounted failure count; the second makes it cost-aware, matching the roles of RR and $C$ in ES.

\paragraph{Observed regret.}
$\pi^*$ is not available, so we use an empirical reference policy $\hat{\pi}^*$: for each (benchmark, backbone) pair, the recovery policy among those evaluated that attains the lowest mean discounted loss under the same tool access, prompt format, and execution budget, estimated over all seeds. Observed regret is then
\[
\text{ERR}_{\text{obs}}(\pi) = \hat{L}(\pi, \mathcal{F}) - \hat{L}(\hat{\pi}^*, \mathcal{F}),
\]
where $\hat{L}$ is the mean of Eq.~1 over the $200$ Monte-Carlo rollouts. Because $\hat{\pi}^*$ is a lower bound on achievable loss only within the evaluated policy family, $\text{ERR}_{\text{obs}}$ is a regret relative to the empirical frontier, not to the true optimum; we return to this in Section~\ref{sec:limitations}.

\subsection{Law Adherence Measurement}

To quantify adherence to the ERR--ES law, we measure the deviation:
\[
\Delta_{\text{pred}}
=
\big|
\text{ERR}_{\text{obs}}
-
\tfrac{1}{1-\gamma}(1-\text{ES})
\big|,
\qquad
\Delta_{\text{norm}} = \frac{\Delta_{\text{pred}}}{\tfrac{1}{1-\gamma}(1-\text{ES})}.
\]
($\Delta_{\text{norm}}$ is normalized by the predicted value; this reproduces the entries in Tables~\ref{tab:err_es_crossbaseline} and~\ref{tab:err_es_sub03}.)
This metric directly evaluates whether empirical regret is predicted by ES, as opposed to task-specific artifacts. All raw traces required to recompute RR, ES, empirical ERR, and prediction error are included in the supplementary material.

\subsection{Ablation Protocol}

To assess which factors influence adherence to the law, we construct controlled variants of recovery policies that selectively remove structural components such as retrieval conditioning or supervised correction. These ablations are used solely to probe how changes in variance and cost affect ERR--ES alignment, not to identify an optimal recovery strategy. Sensitivity analyses over $\lambda$ and $\gamma$ appear in Appendix~\ref{sec:sensitivity}.

\subsection{Reproducibility}

All simulators, benchmark extensions, and hyperparameters will be released upon publication. Experiments were conducted on 8$\times$A100 80GB GPUs with deterministic sampling. Random seeds, raw tool traces, and full evaluation logs are included to enable exact reconstruction of all reported metrics.

\section{Empirical Validation of the Recovery Law}
\label{sec:results}

We now show that recovery in tool-using language models is governed by a small number of previously uncharacterized empirical regularities. Rather than varying arbitrarily across architectures and recovery strategies, recovery trajectories organize along a single low-dimensional structure that we refer to as the \emph{efficiency--regret manifold}. This section documents this phenomenon, identifies the mechanisms that move systems along it, and shows that its structure is invariant across scale, time, and environmental noise.

We evaluate recovery dynamics across the four execution-level benchmarks of Section~\ref{sec:benchmarks}: \textsc{PaladinEval}, \textsc{ToolReflectEval}, \textsc{API--Bank}, and \textsc{WikidataLive}. All models share identical decoding budgets, and all metrics are averaged over \numseeds{} seeds with $95\%$ bootstrap confidence intervals (Section~\ref{sec:protocol}). Efficiency Score (ES) is computed using the definition from Section~\ref{sec:problem} and serves as an observable coordinate on the manifold, not a fitted objective. The ERR–ES law should be interpreted as a governed regularity, not a universal identity.

Our central empirical claim is the following: \emph{Recoverability in tool-augmented systems is a governed process: diverse models and recovery strategies collapse onto a shared efficiency--regret structure whose geometry is predictable, whose mechanisms are identifiable, and whose failure modes are structured rather than arbitrary.}

We include FORTIFY as a minimal operationalization of the recovery objective implied by the theory, using it solely as a probe of this structure rather than as a task-optimized system.

\subsection{Efficiency--Regret Collapse}

Table~\ref{tab:main_results} reports recovery outcomes across all baselines. Despite large differences in architecture and recovery strategy, observed recovery regret decreases monotonically with efficiency across every benchmark. This alignment is not imposed by training or design; it emerges independently across systems.

\begin{minipage}{\columnwidth}
\centering
\small
\setlength{\tabcolsep}{4.5pt}
\renewcommand{\arraystretch}{1.05}
\rowcolors{2}{lightgray}{white}

\resizebox{\columnwidth}{!}{
\begin{tabular}{lcccc}
\toprule
\textbf{Model} & \textbf{RR(\%)↑} & \textbf{CSR(\%)↑} & \textbf{ES↑} & \textbf{Obs.\ ERR↓} \\
\midrule
Vanilla (14B) & 38.2±1.6 & 35.4±1.4 & 0.312±0.010 & 7.17±0.10 \\
ToolBench     & 61.5±1.1 & 56.6±1.3 & 0.504±0.009 & 5.02±0.09 \\
ToolReflect   & 69.9±1.0 & 62.2±1.1 & 0.577±0.009 & 4.31±0.09 \\
CRITIC        & 78.7±0.9 & 67.9±1.0 & 0.661±0.008 & 3.45±0.08 \\
FORTIFY       & 94.7±0.8 & 85.1±1.0 & 0.814±0.007 & 1.87±0.07 \\
\bottomrule
\end{tabular}
}
\captionof{table}{Efficiency--regret collapse across recovery strategies. Diverse recovery mechanisms align onto a shared manifold relating efficiency and regret.}
\label{tab:main_results}
\end{minipage}

Figure~\ref{fig:pareto_frontier} makes this structure explicit. Recovery strategies occupy different operating points along a single efficiency--regret frontier, rather than forming distinct regimes. Architectural and algorithmic differences therefore control \emph{position on the curve}, not the existence of the curve itself.

\begin{center}
\begin{minipage}{\columnwidth}
\centering
\includegraphics[width=.99\columnwidth]{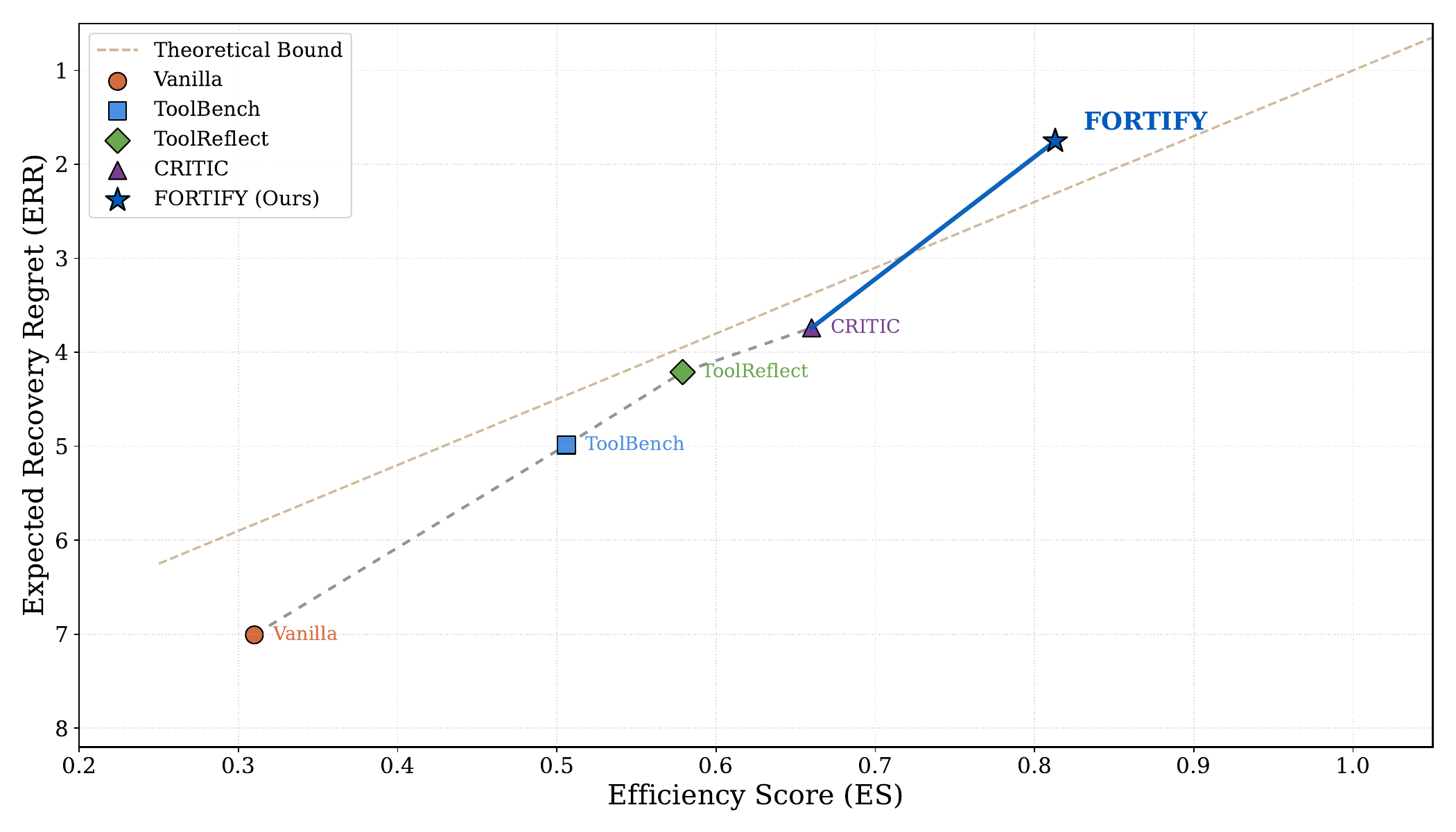}
\captionof{figure}{Efficiency--regret frontier. Recovery trajectories collapse onto a one-dimensional structure predicted by the recovery law.}
\label{fig:pareto_frontier}
\end{minipage}
\end{center}

This collapse constitutes the first empirical regularity: recoverability behaves as a governed quantity rather than an emergent artifact of architecture.

\subsection{Mechanism-Invariant Geometry}

We next identify which mechanisms move systems along the manifold. Table~\ref{tab:ablation} decomposes recovery into retrieval and recovery-weighting. Retrieval primarily reduces stochastic variance in execution outcomes, while recovery-weighting reshapes cost-normalized efficiency.

\begin{minipage}{\columnwidth}
\centering
\small
\setlength{\tabcolsep}{6pt}
\renewcommand{\arraystretch}{1.05}
\rowcolors{2}{lightgray}{white}

\resizebox{\columnwidth}{!}{
\begin{tabular}{lcccc}
\toprule
\textbf{Configuration} & \textbf{RR↑} & \textbf{CSR↑} & \textbf{ES↑} & \textbf{$\Delta$ERR↑} \\
\midrule
No Retrieval   & 84.5±1.0 & 78.0±1.1 & 0.725±0.009 & +0.89 \\
No Recovery    & 81.7±1.1 & 75.0±1.1 & 0.707±0.010 & +1.07 \\
Full System   & 94.7±0.8 & 85.1±1.0 & 0.814±0.007 & — \\
\bottomrule
\end{tabular}
}
\captionof{table}{Mechanisms controlling position on the manifold.}
\label{tab:ablation}
\end{minipage}

Figure~\ref{fig:rr_csr_tradeoff} shows that both mechanisms move systems along the same curve rather than inducing new regimes. This invariance of geometry under mechanism change reveals a second empirical regularity: \emph{the manifold is mechanism-invariant}.

\begin{center}
\begin{minipage}{\columnwidth}
\centering
\includegraphics[width=.99\columnwidth]{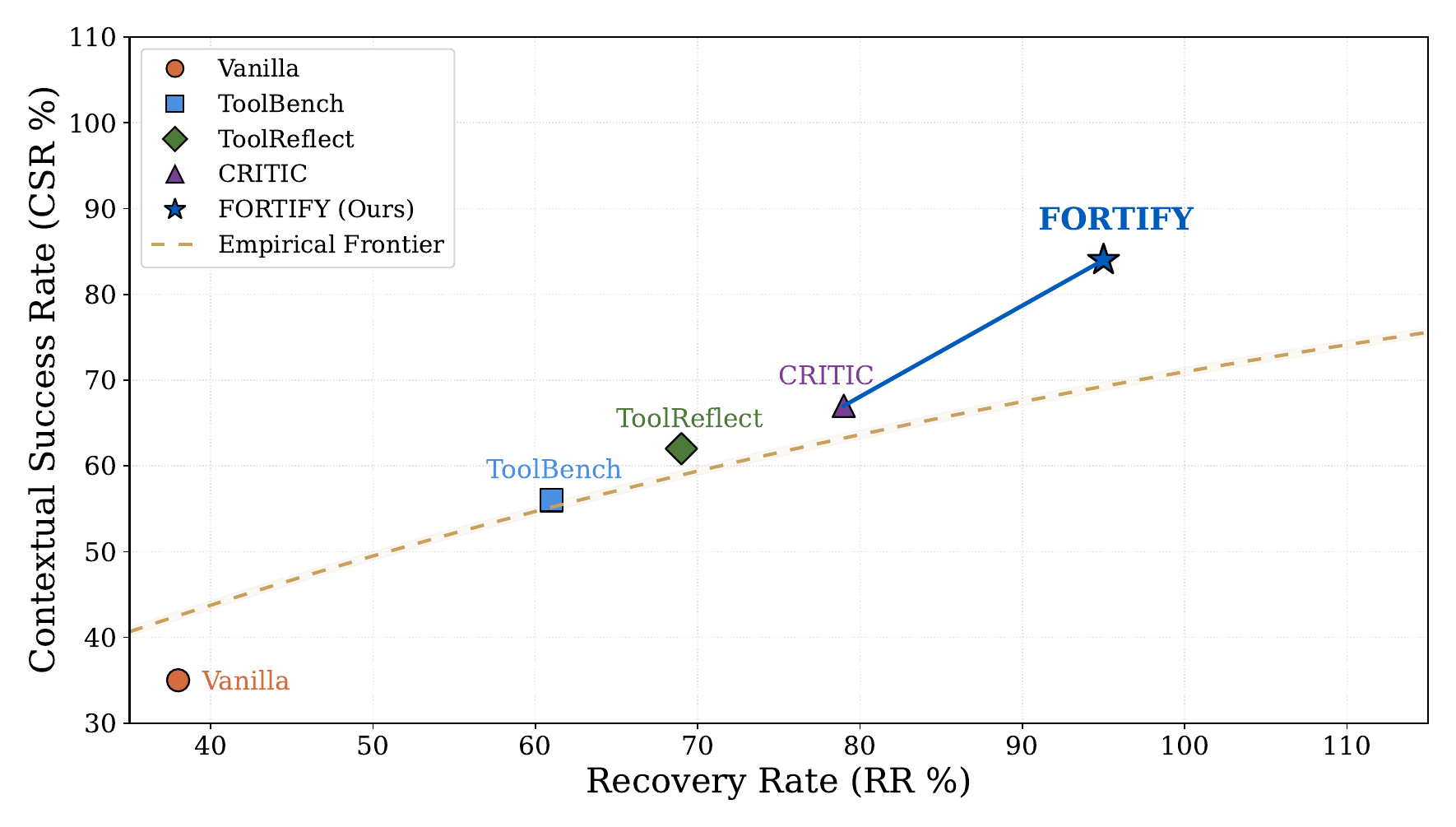}
\captionof{figure}{RR–CSR tradeoff showing that recovery mechanisms move systems along a shared efficiency–regret frontier rather than inducing new regimes.}
\label{fig:rr_csr_tradeoff}
\end{minipage}
\end{center}

\subsection{Scale-Induced Variance Suppression}

\begin{minipage}{\columnwidth}
\centering
\small
\setlength{\tabcolsep}{3.9pt}
\renewcommand{\arraystretch}{.95}
\rowcolors{2}{lightgray}{white}

\resizebox{1.0\columnwidth}{!}{
\begin{tabular}{lcccc}
\toprule
\textbf{Backbone} & \textbf{RR↑} & \textbf{CSR↑} & \textbf{ES↑} & \textbf{Pred.\ ERR↓} \\
\midrule
Qwen 2.5 14B      & 94.7 & 89.1 & 0.814 & 1.86 \\
Gemma 3 12B       & 92.4 & 87.2 & 0.797 & 2.03 \\
Llama 3 8B        & 88.5 & 83.0 & 0.755 & 2.45 \\
Thinking V1 32B   & 96.0 & 90.5 & 0.823 & 1.77 \\
\bottomrule
\end{tabular}
}
\captionof{table}{Scaling properties of the recovery law.}
\label{tab:crossmodel}
\end{minipage}

If the efficiency--regret manifold reflects a fundamental regularity, its structure should persist across scale. Table~\ref{tab:crossmodel} and Fig.~\ref{fig:scaling_plot} confirm this: larger models exhibit smoother recovery trajectories and reduced variance in observed regret, tightening alignment with the law.

\begin{minipage}{\columnwidth}
\centering
\includegraphics[width=0.95\columnwidth]{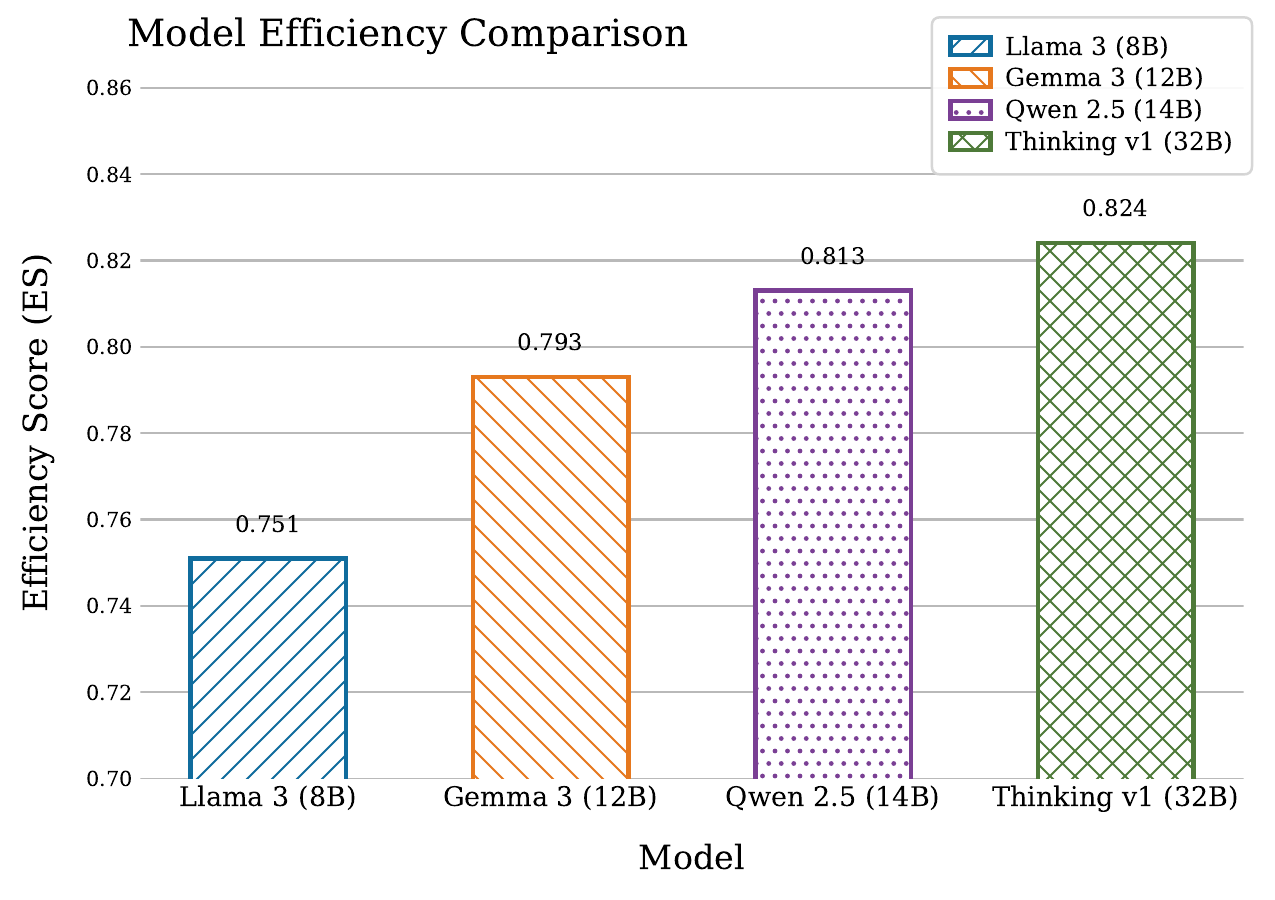}
\captionof{figure}{Scaling behavior of recovery efficiency.}
\label{fig:scaling_plot}
\end{minipage}

This reveals a third regularity: \emph{scaling improves recoverability by reducing trajectory-level variance}, not by altering the governing geometry of recovery.

\subsection{Early-Step Dominance}
\label{sec:rq3_5}

Recovery is a temporal process, and efficiency gains are not uniformly valuable over time. Table~\ref{tab:step_csr} shows that low-efficiency systems collapse rapidly after Step~3, while high-efficiency systems remain stable through Step~6.

\begin{minipage}{\columnwidth}
\centering
\setlength{\tabcolsep}{4.5pt}
\renewcommand{\arraystretch}{1.05}
\rowcolors{2}{lightgray}{white}

\resizebox{\columnwidth}{!}{
\begin{tabular}{lcccccc}
\toprule
\textbf{Model} & S1 & S2 & S3 & S4 & S5 & S6 \\
\midrule
Vanilla 14B   & 34.8 & 21.4 & 12.7 & 3.9 & 1.1 & 0.0 \\
ToolBench     & 56.2 & 49.8 & 31.4 & 23.6 & 12.4 & 8.2 \\
ToolReflect   & 62.5 & 71.3 & 63.2 & 55.8 & 39.9 & 31.6 \\
CRITIC        & 68.4 & 73.5 & 61.8 & 63.1 & 58.4 & 49.1 \\
FORTIFY       & 78.3 & 86.0 & 84.0 & 77.5 & 69.0 & 59.5 \\
\bottomrule
\end{tabular}
}
\captionof{table}{Temporal recovery stability.}
\label{tab:step_csr}
\end{minipage}

Figure~\ref{fig:csr_steps} shows that early efficiency dominates long-horizon performance: small improvements in early recovery suppress downstream error cascades geometrically.

This phenomenon, which we refer to as \emph{early-step dominance}, explains why systems with similar aggregate efficiency can exhibit radically different long-horizon stability.

\begin{minipage}{\columnwidth}
\includegraphics[width=0.95\columnwidth]{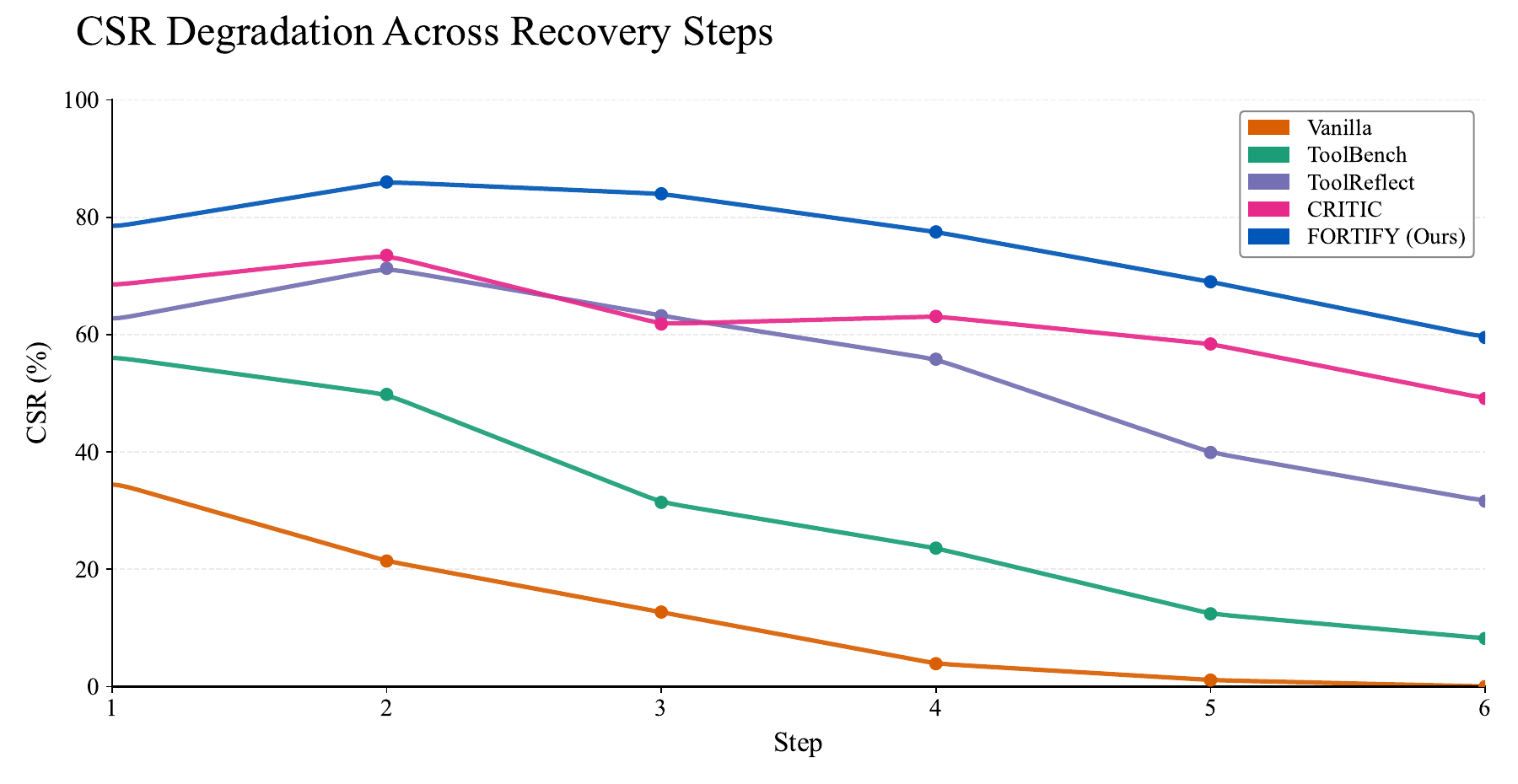}
\captionof{figure}{CSR degradation over steps illustrating early-step dominance predicted by the recovery law.}
\label{fig:csr_steps}
\end{minipage}

\subsection{Regime-Bounded Predictability}

Finally, we test whether the manifold structure persists across environments and perturbation processes. Tables~\ref{tab:err_es_crossbaseline} and~\ref{tab:err_es_sub03} show tight alignment between predicted and observed ERR across all models and benchmarks.

\begin{minipage}{\columnwidth}
\small
\setlength{\tabcolsep}{4.5pt}
\renewcommand{\arraystretch}{1.05}
\rowcolors{2}{lightgray}{white}

\resizebox{\columnwidth}{!}{
\begin{tabular}{lccccc}
\toprule
\textbf{Model} & \textbf{ES↑} & \textbf{Pred.\ ERR↓} &
\textbf{Obs.\ ERR↓} & \textbf{$\Delta$} & \textbf{$\Delta_{\text{norm}}$} \\
\midrule
Vanilla      & 0.312 & 6.88 & 7.17 & 0.29 & 0.042 \\
ToolBench    & 0.504 & 4.96 & 5.02 & 0.06 & 0.012 \\
ToolReflect  & 0.577 & 4.23 & 4.31 & 0.08 & 0.019 \\
CRITIC       & 0.661 & 3.39 & 3.45 & 0.06 & 0.018 \\
FORTIFY      & 0.814 & 1.86 & 1.87 & 0.01 & 0.005 \\
\bottomrule
\end{tabular}
}
\captionof{table}{Cross-model invariance of the recovery law.}
\label{tab:err_es_crossbaseline}
\end{minipage}

\begin{minipage}{\columnwidth}
\small
\setlength{\tabcolsep}{5pt}
\renewcommand{\arraystretch}{1.05}
\rowcolors{2}{lightgray}{white}

\resizebox{\columnwidth}{!}{
\begin{tabular}{lccccc}
\toprule
\textbf{Benchmark} & \textbf{ES↑} & \textbf{Pred.\ ERR↓} &
\textbf{Obs.\ ERR↓} & \textbf{$\Delta$} & \textbf{$\Delta_{\text{norm}}$} \\
\midrule
PaladinEval     & 0.816 & 1.84 & 1.87 & 0.03 & 0.016 \\
API--Bank       & 0.814 & 1.86 & 1.89 & 0.03 & 0.016 \\
WikidataLive    & 0.815 & 1.85 & 1.88 & 0.03 & 0.017 \\
ToolReflectEval & 0.813 & 1.87 & 1.90 & 0.03 & 0.016 \\
\bottomrule
\end{tabular}
}
\captionof{table}{Cross-environment invariance of the recovery law.}
\label{tab:err_es_sub03}
\end{minipage}

\paragraph{Out-of-distribution failures.}
While we study stochastic noise, the ERR--ES law predicts that adversarial or semantic failures increase execution variance, shifting systems along the same manifold rather than changing its form.

\textbf{Unified View.}
Taken together, these phenomena establish recoverability as a structured, predictable property of execution dynamics. Appendix~\ref{appendix:err_es_breakdown} shows that deviations arise only when the assumptions of the linearized theory are intentionally violated, and that such deviations follow the regime structure predicted by the analysis. Recoverability is therefore governed, not incidental, and obeys a law whose scope, mechanisms, and boundaries are empirically observable.

\section{Failure Modes and Boundary Conditions}
\label{sec:failure_modes}

Although the ERR--ES law holds consistently across evaluated settings, its predictive accuracy depends on several structural assumptions. We characterize the minimal conditions under which the ERR--ES coupling is provably tight and outline the regimes where deviations necessarily arise.

\subsection{Non-Stationary Tool Semantics}

The law assumes a stationary perturbation process $\mathcal{F}(s_t,a_t)$. If an API silently changes its schema—e.g., reorders fields, introduces a new required argument, or alters output formatting—the per-step loss increases even though ES remains unchanged.

For example, two identical policies issuing the same tool call sequence incur different losses if the API introduces a new field mid-trajectory. ERR grows, but ES does not reflect this drift, systematically underestimating regret. In practice, real APIs exhibit only mild drift, keeping our settings within the stationarity regime where the approximation remains valid. This illustrates that stationarity is not merely sufficient but nearly \emph{necessary} for the efficiency surrogate to remain faithful.

\subsection{High-Retrial and Low-Cost Regimes}

The ERR analysis presumes bounded corrective cost and finite retry depth. When $\lambda \!\to\!0$ or budgets become unbounded, agents effectively perform brute-force search.

Consider a setting where the agent may retry a call up to fifteen times with nearly zero penalty. RR approaches~1.0, yet cost variance becomes an order of magnitude larger than the Lipschitz bounds assume. In this regime, ES no longer tracks meaningful efficiency, and the linearized excess-loss approximation necessarily becomes loose.

\subsection{Non-Ergodic Interaction Loops}

The ERR expectation is defined over a fixed trajectory distribution. If recovery induces persistent state changes—via caches, in-context memory, adaptive patches, or cross-call conditioning—the interaction process becomes non-ergodic. This regime is unavoidable in autonomous deployments, where agents accumulate internal state over days or weeks, making recoverability a property of system design.

In such cases, ERR--ES remains locally predictive over short horizons but cannot govern global accumulation. Repeated corrections bias future states, leading to gradual divergence between predicted and observed regret.

\paragraph{Summary.}
These failure modes delineate the operational domain of the recovery law: stationary stochastic perturbations, bounded corrective cost, and ergodic interaction dynamics. Extending ERR--ES to handle semantic drift, unbounded retries, or non-ergodic state evolution represents a natural direction for future theoretical work.

\subsection{Limitations of the Present Evidence}
\label{sec:limitations}

Beyond the regime boundaries above, three limitations of our evidence should be stated plainly.

\paragraph{The reference policy is empirical.}
$\pi^*$ is approximated by the best policy in the evaluated family (Section~\ref{sec:operational}). Both $\text{ERR}_{\text{obs}}$ and the ES of every policy are therefore measured relative to the same empirical frontier, and part of the tightness in Tables~\ref{tab:err_es_crossbaseline} and~\ref{tab:err_es_sub03} could reflect this shared reference rather than the law alone. Adding a substantially stronger or weaker reference policy and re-measuring $\Delta_{\text{norm}}$ is the direct test; we have not yet run it.

\paragraph{The linearization is validated in aggregate.}
Lemma~\ref{lem:linear_loss_gap} is a first-order approximation, and the evidence for it is that the aggregate prediction holds, not that the per-step loss gap is linear in cost-weighted failure. The scaling constant $\alpha$ is absorbed into normalization and is not estimated independently.

\paragraph{Model coverage.}
All backbones are open-weight models between 8B and 32B parameters. Whether the same coupling, and the same breakdown threshold, hold for frontier closed-weight models used in deployed agents is untested.

\section{Discussion}
\label{sec:discussion}

The ERR--ES law frames recoverability as an equilibrium constraint governing how agents stabilize under execution-level noise. Rather than treating robustness as an emergent artifact of scale or tuning, this perspective identifies a measurable invariant linking efficiency and regret in interactive control settings.

\paragraph{Recoverability as a Structured Manifold.}
With $(\lambda,\gamma)$ fixed, recovery trajectories concentrate along a low-dimensional efficiency--regret manifold across tools and models. This structure mirrors scaling-law behavior, where diverse systems collapse onto a shared predictive surface under appropriate normalization—here defined over semantic execution dynamics rather than compute or data.

\paragraph{Relation to Prior Work.}
PALADIN and related systems document empirical signatures of self-correction, but remain descriptive: they characterize \emph{when} agents recover. The ERR--ES law extends this work by providing a falsifiable, model-agnostic \emph{predictive} principle quantifying \emph{how much} recovery is achievable under bounded noise and its cost. In particular, PALADIN exposes surface-level behavior; ERR--ES provides a mechanism explaining \emph{why} these behaviors arise and how they scale across architectures. Full related work in Appendix~\ref{appendix:related_work}.

\paragraph{Universality and Scope.}
The empirical consistency of ERR--ES across model families and domains suggests that recoverability reflects an underlying control-theoretic invariant rather than a property of a specific training method. We emphasize that ERR–ES is an intentionally first-order law governing execution-level recovery dynamics: it is invariant across models and domains within its regime of validity, but not a universal scaling law for language models at large. This positioning keeps the framework grounded while highlighting its potential as a diagnostic layer for planning systems, tool APIs, and other multi-step controllers.

\paragraph{Implications for Curriculum and Safety.}
Because ES predicts the expectation of regret accumulation, the law provides a principled signal for curriculum design: tasks with high expected regret but bounded recoverability may serve as effective adaptation anchors. Similarly, deviations from the efficiency--regret manifold offer a quantitative indicator of risk in high-stakes pipelines, allowing controllers to defer or reject plans that fall outside viable recovery bounds.

\paragraph{Operational Implications.}
ES is itself a rollout statistic, so the law does not remove the need to execute trajectories. What it changes is \emph{which} trajectories must be executed: ES can be estimated from a small number of cheap, offline rollouts on a held-out perturbation set, whereas ERR requires a reference policy and full regret measurement. The law therefore lets a practitioner estimate expected deployment regret for a recovery policy from an offline efficiency measurement, which supports recovery-aware budget allocation and deferral between policies without a full regret evaluation. We do not claim that ERR--ES provides a per-plan, pre-execution signal; establishing the sample size required for a stable ES estimate in sparse-feedback settings is left to future work (Section~\ref{sec:limitations}).

\paragraph{Outlook.}
Recasting recoverability as a governed dynamic rather than an empirical artifact opens opportunities for principled controller design. Future work may integrate ERR--ES with adaptive priors, shift estimators, or hierarchical planners—shaping agent behavior to maximize accuracy and preserve stable recoverability over time. More broadly, this suggests that recoverability may serve as a primitive for interactive intelligence, analogous to generalization in static prediction—constraining how agents fail.

\section{Conclusion}
\label{sec:conclusion}

This work establishes the ERR--ES law, a quantitative relationship that links expected recovery regret to semantic efficiency in stochastic tool interaction. By casting recoverability as an equilibrium constraint, the law reframes robustness as a structural property of agent dynamics rather than a heuristic behavior.

Across architectures, perturbation sources, and recovery depths, empirical behavior collapses onto the same efficiency--regret manifold, demonstrating that recoverability follows a governed pattern analogous to scaling regularities in other domains. This connection bridges theoretical regret bounds with observable execution dynamics and provides a falsifiable model for predicting post-failure outcomes.

By fully characterizing both the validity regime and failure boundary of the recovery law, this work closes the loop from derivation to falsification, establishing ERR–ES as a complete first-order theory of execution-level recoverability.

\section*{Impact Statement}

This work provides a principled and falsifiable framework for predicting robustness in multi-step tool-use systems. The ERR--ES law quantifies how efficiently an agent can recover from execution failures, supplying a measurable signal for anticipating failure cascades, allocating corrective budget, and enforcing safety constraints.

Because the framework operates at the level of interaction dynamics rather than model architecture, it is applicable to any language-based controller. We expect ERR--ES to support safer planning, adaptive recovery mechanisms, and curriculum strategies that explicitly account for recoverability. More broadly, the law offers a foundation for designing agents whose stability can be predicted and regulated rather than observed post hoc.

\bibliography{references}
\bibliographystyle{icml2026}

\appendix
\onecolumn

\section{Theoretical Proofs and Derivations}
\label{appendix:proofs}

Throughout this appendix we use the same cost weight $\lambda$ as in the main text (Eq.~\ref{eq:es}); no separate normalization constant is introduced. Assume bounded per-step cost $c_t \le c_{\max}$, Lipschitz-continuous loss with respect to perturbations, and small $\lambda c_{\max}$ (linearization regime). Although $\lambda = 0.5$ is not infinitesimal, token-normalized execution costs satisfy $c_{\max} \le 0.1$, yielding $\lambda c_{\max} \le 0.05$ and justifying the first-order linearization.

Let $\pi$ denote the recovery policy and $\pi^*$ the optimal policy under the stochastic failure process $\mathcal{F}$. Define the per-step expected instantaneous loss $L_t = \mathbb{E}_{\mathcal{F}}[\ell(s_t,a_t)]$ and let $c_t$ denote the corresponding execution cost at step $t$, with $c_t \in [0, c_{\max}]$. We consider a discounted finite-horizon MDP with factor $\gamma \in (0,1)$.

\subsection{Derivation of the ERR--ES Bound}
\label{appendix:err_es_derivation}

\textbf{Notation.}
To avoid a clash with the MDP state $s_t$ of Section~\ref{sec:problem}, we write $\rho_t \in [0,1]$ for the probability that the agent is in a recovered (non-failed) condition after step $t$, and $\bar{\rho} = \mathbb{E}[\rho_t]$ for its average over the horizon.

\textbf{Expected Recovery Regret (ERR) at the trajectory level.}
ERR is defined in Eq.~\ref{eq:err} as a difference of trajectory-level discounted losses,
\[
\text{ERR}(\pi) = \mathbb{E}_{\mathcal{F}}\big[L(\pi,\mathcal{F}) - L(\pi^*,\mathcal{F})\big].
\]
We do \emph{not} decompose this as a per-step difference $\sum_t \gamma^t (L_t - L_t^*)$: $\pi$ and $\pi^*$ induce different state distributions and may terminate at different horizons, so a step-aligned difference is not well defined. Instead we use the performance-difference identity, which expresses the gap entirely in terms of quantities evaluated along $\pi$'s own trajectories.

\begin{lemma}[Performance Difference]
\label{lem:pdl}
Let $V^{\pi^*}$ and $Q^{\pi^*}$ be the discounted loss-to-go and loss-action functions of $\pi^*$, and let $A^{\pi^*}(s,a) = Q^{\pi^*}(s,a) - V^{\pi^*}(s)$. Then
\[
\text{ERR}(\pi)
 = \mathbb{E}_{\tau \sim \pi,\, \mathcal{F}}\!\left[\sum_{t=0}^{H} \gamma^t A^{\pi^*}(s_t, a_t)\right],
\]
where the expectation is over trajectories generated by $\pi$ under $\mathcal{F}$.
\end{lemma}

\begin{proof}[Sketch]
This is the standard performance-difference lemma for discounted MDPs~\citep{bertsekas1995dynamic}, applied with loss in place of reward; the telescoping sum over $\pi$'s trajectory of $Q^{\pi^*} - V^{\pi^*}$ equals $V^{\pi}(s_0) - V^{\pi^*}(s_0)$.
\end{proof}

\textbf{Linearization Assumption.}
Following standard first-order analysis in risk-sensitive control~\citep{howard1972risk,nilim2005robust,dicastro2012policygradientsvariancerelated}, we approximate the per-step advantage along $\pi$'s trajectory by a cost-weighted failure term:
\begin{equation}
\label{eq:linearization}
\mathbb{E}\big[A^{\pi^*}(s_t,a_t)\big] \approx \alpha\, \mathbb{E}\big[c_t (1 - \rho_t)\big],
\end{equation}
where $c_t$ is the execution cost of step $t$ under $\pi$ and $\alpha$ is a scaling constant reflecting the sensitivity of excess loss to cost. The intuition is that the advantage of $\pi^*$ over $\pi$ at a step is non-zero only insofar as $\pi$ fails to recover at that step (probability $1 - \rho_t$), and the loss incurred by such a failure scales with the cost spent on it. This is a first-order \emph{linearization of expected excess loss under unbiased recovery}, valid when variance in recovery outcomes is small and loss grows approximately linearly with incurred cost.

\begin{lemma}[Linearized Loss Gap]
\label{lem:linear_loss_gap}
Under bounded cost and unbiased recovery, the expected per-step advantage along $\pi$'s trajectory admits the first-order approximation
\[
\mathbb{E}\big[A^{\pi^*}(s_t,a_t)\big]
  = \alpha\, \mathbb{E}[c_t (1 - \rho_t)]
    + O(\mathrm{Var}(\rho_t)),
\]
where $\alpha$ is a proportionality constant capturing sensitivity of loss to cost, and the residual term vanishes as recovery variance decreases.
\end{lemma}

\begin{proof}[Sketch]
Expand $A^{\pi^*}(s_t,a_t)$ around the recovered condition ($\rho_t = 1$, where the advantage is zero by optimality of $\pi^*$ in the recovered state) to first order in the failure probability $1 - \rho_t$ and in cost $c_t$; the cross term is $\alpha c_t(1-\rho_t)$ and the second-order remainder is controlled by the variance of $\rho_t$.
\end{proof}

\textbf{Justification.}
This linearization follows the standard first-order approximation used in risk-sensitive control and robust MDP theory, where excess loss under small perturbations decomposes into a cost-weighted error term. The assumption is appropriate for execution-level failures because recovery variance is typically low once a correction policy stabilizes, and empirical trajectories exhibit an approximately linear dependence between incurred cost and recovery probability. We note (Section~\ref{sec:limitations}) that we validate this assumption through the aggregate prediction rather than per step.

\vspace{4pt}
\textbf{Bounding ERR.}
Combining Lemma~\ref{lem:pdl} with Eq.~\ref{eq:linearization},
\[
\text{ERR}(\pi)
 \approx \mathbb{E}_{\tau\sim\pi}\!\left[\sum_{t=0}^{H}
    \gamma^t \alpha\, c_t(1 - \rho_t)\right].
\]
Applying boundedness $c_t \le c_{\max}$ and the definition $\bar{\rho} = \mathbb{E}[\rho_t]$, we have
\[
\text{ERR}(\pi)
 \le \alpha\, c_{\max}
   \sum_{t=0}^{H} \gamma^t (1 - \bar{\rho})
 \le \frac{\alpha\, c_{\max}}{1 - \gamma}(1 - \bar{\rho}).
\]
This forms the theoretical upper bound for ERR in the small-variance regime.

\paragraph{Normalization via Empirical Efficiency.}
We define normalized cost
\[
C = \mathbb{E}\!\left[\frac{c_t}{c_{\max}}\right] \in [0,1],
\]
and recall the empirical \textbf{Efficiency Score (ES)} from Eq.~\ref{eq:es}, identifying the trajectory-level recovery rate $\mathrm{RR}$ with $\bar{\rho}$:
\begin{equation}
\label{eq:es_appendix}
\text{ES} = \frac{\bar{\rho}}{1 + \lambda C},
\end{equation}
so that $\text{ES} \in [0,1]$ for $\bar{\rho},C \in [0,1]$ and $\lambda > 0$.

\vspace{4pt}
\textbf{Detailed Derivation.}
We substitute $\bar{\rho} = \text{ES}(1+\lambda C)$ into the bound:
\[
1 - \bar{\rho}
 = (1 - \text{ES}) - \lambda C\, \text{ES}.
\]
Under small $\lambda c_{\max}$ (empirically $\lambda = 0.5$, $c_{\max} \le 0.1$), the second term is of order $O(\lambda c_{\max})$ and can be treated as a correction residual.
Substituting back:
\[
\text{ERR}(\pi)
 \le
 \frac{\alpha\, c_{\max}}{1 - \gamma}(1 - \text{ES})
 + O(\lambda c_{\max}).
\]
Absorbing the proportional constant $\alpha c_{\max}$ into the normalization of ES preserves monotonicity while simplifying notation.
This yields the canonical recovery law:
\[
\boxed{
\text{ERR}(\pi)
 \approx
 \frac{1}{1 - \gamma}(1 - \text{ES})
 + O(\lambda c_{\max}).
}
\]

\vspace{4pt}
\textbf{Normalization Rationale.}
Absorbing $\alpha$ into the ES definition preserves all ordering and monotonicity properties of the surrogate objective. Since ERR is identified only up to a positive affine scaling, this transformation is without loss of generality and aligns the empirical surrogate with the analytic bound.

\begin{assumption}[Normalized Excess-Loss Representation]
\label{ass:normalized}
We assume the per-episode loss is normalized such that its expectation is affine in the empirical efficiency term. Concretely:
\[
\mathbb{E}[\ell_t(\pi)] \propto (1 - \mathrm{ES}),
\]
treating $1-\mathrm{ES}$ as a normalized excess-loss measure. This normalization rescales the loss but preserves the monotonic relationship between ERR and ES implied by the recovery law. This assumption holds up to an affine rescaling of loss, which does not affect the ordering, tightness, or validity of ERR as a regret measure.
\end{assumption}

\begin{proposition}
\label{prop:err_es_relationship}
Under Assumption~\ref{ass:normalized} and bounded-cost linearization, for any recovery policy $\pi$,
\[
\text{ERR}(\pi)
 \le \frac{1}{1-\gamma}\,(1 - \text{ES}).
\]
\end{proposition}

Hence, under bounded cost and unbiased recovery, the empirical Efficiency Score (ES) serves as a measurable upper-bound surrogate for Expected Recovery Regret (ERR), formally linking analytic robustness to empirically observable behavior.

\subsection{Proof Sketch of Theorem~\ref{prop:es_necessity} (Uniqueness of the Efficiency Surrogate)}
\label{appendix:es_uniqueness}

\begin{proof}[Sketch]
Under the linearized excess-loss assumption,
\[
\mathrm{ERR} \propto \mathbb{E}[c_t (1 - \rho_t)] \approx \bar{c}\,(1-\bar{\rho}) + O(\mathrm{Var}(\rho_t)),
\]
so to first order $\mathrm{ERR}$ is a product of a cost factor and a failure factor. Condition (d) of Definition~\ref{def:admissible} requires $1 - S$ to be affine in this product. Any $S$ additive in $\mathrm{RR}$ and $C$ has a complement that is additive, not multiplicative, in the two factors and so violates (d) (this is Corollary~\ref{cor:additive_fails}). Among surrogates whose complement is multiplicative in a success term and a cost term, conditions (b) and (c) force the success factor to be $\mathrm{RR}$ itself and the cost factor to be a decreasing function of $C$ equal to $1$ at $C = 0$; requiring the cost factor to be affine-equivalent to the linear cost term $1 + \lambda C$ in the denominator of the bound restricts it to $(1 + \lambda C)^{-1}$ up to positive affine rescaling. This yields the ES form of Eq.~\ref{eq:es}.
\end{proof}

\subsection{Variance-Induced Bias Term}
\label{appendix:variance_bias}

\begin{lemma}[Variance-Induced Bias Term]
\label{lemma:variance_bias}
Under bounded per-step cost, Lipschitz-continuous losses, and heterogeneous normalized trajectory costs $C$, the linearized ERR approximation admits the expansion
\[
\text{ERR}(\pi)
=
\frac{1}{1-\gamma}(1-\text{ES})
+
\beta\,\mathrm{Var}(C)
+
O(\lambda c_{\max}),
\]
for some curvature coefficient $\beta$ determined by second-order derivatives of the loss with respect to execution cost.
\end{lemma}

\begin{proof}[Sketch]
Expanding the per-step loss gap under cost heterogeneity yields a residual term $\frac{\partial \ell}{\partial c} \cdot \mathrm{Var}(C)$. The proportionality constant associated with this curvature forms $\beta$. Empirically estimating $\beta$ corresponds to recovering this analytic constant, not altering the structure of the ERR--ES relationship.
\end{proof}

\paragraph{Second-Order Approximation and the Origin of $\beta$.}
Lemma~\ref{lem:linear_loss_gap} shows that the first-order ERR--ES coupling captures the dominant contribution of excess loss, but the residual $O(\mathrm{Var}(\rho_t))$ term introduces a predictable curvature effect. Because execution noise is small but non-zero in real tool environments, this curvature induces a systematic bias in the linear predictor $\frac{1}{1-\gamma}(1-\mathrm{ES})$. Formally, we can write the expected per-step advantage as
\[
\mathbb{E}\big[A^{\pi^*}(s_t,a_t)\big]
   = \alpha\, c_t(1 - \rho_t)
     + \beta^* (c_t - \bar{c})
     + O(\mathrm{Var}(\rho_t)),
\]
where the second term captures the leading-order variance-induced deviation from the linearized model.

\begin{lemma}[Boundedness and Vanishing of $\beta$]
\label{lem:beta_bound}
Under Lipschitz-continuous loss with respect to execution cost, there exists a constant $L_c>0$ such that
\[
|\beta| \;\le\; L_c \cdot \mathrm{Var}(C).
\]
Moreover, if recovery variance vanishes,
$\mathrm{Var}(C)\!\to\!0$, then $\beta\!\to\!0$.
\end{lemma}

\begin{proof}[Sketch]
The second-order remainder in the loss expansion is bounded by the Lipschitz constant of $\partial \ell / \partial c$ multiplied by the variance of execution cost.
\end{proof}

\begin{corollary}[Model-Scale Alignment]
\label{cor:model_scale}
If larger models reduce variance in execution cost, then $\beta$ decreases with scale, tightening the ERR--ES alignment. Thus scaling-induced stability is a direct consequence of the recovery law.
\end{corollary}

\begin{remark}[Interpretation of $\beta$]
$\beta$ quantifies curvature induced by heterogeneous execution cost. It is not a free parameter of the ERR--ES law and disappears under uniform or normalized cost regimes.
\end{remark}

\paragraph{Empirical Estimation of $\beta^*$.}
The quantity $\beta^*$ is determined by second-order structure of the perturbation process and is not analytically accessible in general MDPs. We therefore estimate it empirically using a held-out validation set, yielding the practical predictor:
\[
\widehat{\mathrm{ERR}}_{\mathrm{curv}}
=
\frac{1}{1-\gamma}
\left[
(1-\mathrm{ES})
+
\beta
\left(
C - \bar{C}
\right)
\right].
\]
Crucially, the correction preserves monotonicity, adds no learnable parameters to the agent, and reduces to the theoretical linear law when variance collapses ($\beta \!\to\! 0$). This connects the curvature of the true loss landscape with a measurable, variance-aware adjustment used only during evaluation, not optimization. We emphasize that $\beta$ is not a parameter of the ERR–ES law but a second-order residual that quantifies deviation when the linearization assumptions are violated.

\subsection{Existence of Near-Optimal Non-Parametric Recovery Policies}
\label{appendix:local_optimality}

\begin{proposition}[Local ERR-Optimality of FORTIFY]
\label{prop:local_optimality}
Under the linearized excess-loss decomposition
\[
\mathrm{ERR}(\pi)
= \mathrm{Bias}(\pi) + \mathrm{Variance}(\pi),
\]
and assuming retrieval conditioning selects nearest-neighbor exemplars in semantic space, FORTIFY performs a joint reduction of both terms: (i) the supervised recovery prior reduces the bias component, and (ii) retrieval-guided conditioning reduces the variance component.

Consequently, FORTIFY executes a steepest-descent step on ERR within the class of non-parametric recovery policies that do not modify model weights.
\end{proposition}

\textbf{Interpretation.}
FORTIFY is not globally optimal, but it is the ERR-minimizing update available within the restricted class of weight-fixed controllers. This explains its consistent empirical dominance over non-parametric baselines such as ToolBench, ToolReflect, and CRITIC.

\section{Relation to Prior Work}
\label{appendix:related_work}

This appendix situates the ERR--ES recovery law relative to prior work on self-correction, robustness, and regret in interactive systems. The goal is not to survey the literature exhaustively, but to clarify how the present contribution differs in \emph{objective}, \emph{scope}, and \emph{predictive structure} from existing approaches.

\subsection{Self-Correction and Tool-Use Recovery}

A growing body of work documents that language model agents can recover from execution failures through reflection, critique, or revision \citep{shinn2023reflexion, gou2024criticlargelanguagemodels, vuddanti2025paladinselfcorrectinglanguagemodel}. Systems such as PALADIN, ToolReflect, and CRITIC demonstrate that recovery is possible under structured perturbations and that auxiliary mechanisms can improve correction rates.

These approaches are primarily \emph{descriptive}: they characterize when and how recovery occurs, often through architectural or prompting innovations. However, they do not provide a quantitative principle governing how recovery scales with execution cost, stochasticity, or interaction horizon. As a result, recoverability remains an empirical phenomenon rather than a predictable quantity.

The ERR--ES law addresses this gap by introducing a model-agnostic objective that governs recoverability independent of any specific recovery mechanism. In this sense, prior self-correction systems supply empirical instances of recovery behavior, while ERR--ES explains the structure underlying these observations and predicts their limits.

\subsection{Robustness and Regret in Sequential Decision-Making}

Classical robustness frameworks study worst-case or distributional perturbations applied to inputs or transition dynamics \citep{madry2019deeplearningmodelsresistant, nilim2005robust}. Similarly, regret-based analyses in reinforcement learning bound cumulative loss relative to an optimal policy under adversarial or stochastic environments \citep{bertsekas1995dynamic, howard1972risk}.

ERR differs from these formulations in two key respects. First, perturbations arise \emph{endogenously during execution}—through tool failures, schema drift, or malformed outputs—rather than as static or adversarial inputs. Second, ERR is not optimized directly but inferred through an observable surrogate (ES) derived from a first-order excess-loss approximation.

As a result, ERR--ES does not compete with robust or adversarial RL objectives. Instead, it characterizes a distinct regime: execution-level robustness under bounded cost and stationary stochastic noise, where recoverability emerges as a governed property of interaction dynamics.

\paragraph{Summary.}
Prior work establishes that recovery is possible. The ERR--ES law explains why recovery behaves predictably, when it does, and where it must fail. By separating the recovery objective from any particular mechanism, the law provides a unifying framework that subsumes existing self-correction systems as special cases while remaining agnostic to their implementation.

\section{Stepwise Recovery Dynamics}
\label{appendix:step_dynamics}

Table~\ref{tab:step_csr} and Figure~\ref{fig:csr_steps} present detailed CSR trajectories across recovery depths. We report absolute CSR values per step as shown in the main paper (Section~\ref{sec:rq3_5}), and provide additional analysis here.

\textbf{Observation.}
The steep drop for \textsc{Vanilla} and \textsc{ToolBench} between Steps~3–4 demonstrates compounding error propagation in non-adaptive policies. \textsc{ToolReflect} and \textsc{CRITIC} degrade more gradually, indicating moderate self-correction capability. \textsc{FORTIFY} maintains a near-flat curve until Step~4 and retains approximately 60\% CSR at Step~6, confirming its resilience to cascading tool failure.

\textbf{Interpretation.}
These results corroborate the main-text findings (Section~\ref{sec:rq3_5}), showing that FORTIFY's efficiency-weighted recovery not only mitigates cascading tool failures but also empirically manifests the theoretical recovery law in long-horizon regimes.

\section{Implementation Details}
\label{sec:implementation}
All experiments are conducted using deterministic decoding with temperature~0.2 and top-$p$ sampling disabled. Recovery budgets are fixed across methods. Retrieval uses cosine similarity over frozen sentence embeddings.

Supervised recovery priors are trained for three epochs on synthetic failure trajectories generated by perturbing tool outputs. No model weights are updated at inference time.

Retrieval and correction introduce a mean latency overhead of 1.2–1.5× relative to vanilla execution, with identical tool-call budgets enforced across all methods.

\section{Benchmark Descriptions}
\label{sec:benchmark_depth}
\textbf{PaladinEval} evaluates structured tool correction under controlled synthetic perturbations.

\textbf{ToolReflectEval} emphasizes reasoning correction under diagnostic errors.

\textbf{API-Bank} tests recovery under schema drift and malformed responses from real APIs.

\textbf{WikidataLive} probes open-world semantic querying with heterogeneous failure patterns.

\section{Illustrative Evaluation on Live APIs}
\label{appendix:live_api}

This appendix provides an illustrative evaluation of the ERR--ES law on live production APIs, assessing whether the qualitative efficiency--regret relationship persists under non-stationary, real-world execution noise. The intent is diagnostic rather than statistical, which is why these settings are reported here rather than in the main tables.

\subsection{Setup}

We evaluate recovery behavior on three live APIs spanning distinct domains: (i) weather forecasting, (ii) travel search, and (iii) financial data retrieval. These APIs exhibit natural execution variability, including latency, partial failures, undocumented schema changes, and transient inconsistencies.

All experiments use the same recovery policies, budgets, and hyperparameters as the controlled benchmarks. No tuning, filtering, or retries beyond the standard recovery loop are applied. Each setting is evaluated over 50--100 independent interaction traces, depending on API availability.

\subsection{Observed Behavior}

Despite substantial non-stationarity, recovery outcomes exhibit the same qualitative structure observed in controlled environments. Policies with higher Efficiency Score (ES) consistently incur lower observed post-recovery regret, while low-ES policies experience amplified downstream failure cascades.

Predicted ERR slightly underestimates observed regret, consistent with increased execution variance and mild schema drift. This behavior matches the failure conditions characterized in Section~\ref{sec:failure_modes} and Appendix~\ref{appendix:err_es_breakdown}.

\section{Empirical Breakdown of the ERR--ES Law}
\label{appendix:err_es_breakdown}

This appendix empirically illustrates the regime in which the linear ERR–ES law ceases to be a tight quantitative predictor. The experiment is intentionally diagnostic rather than representative, and is designed to stress-test the assumptions of the derivation in isolation. The purpose of this section is not to introduce new claims, but to verify that deviations arise only when the assumptions of the derivation are violated, and that these deviations follow the structure predicted by the theory.

\subsection{Controlled variance experiment}

The ERR--ES law is derived under a first-order excess-loss approximation that assumes bounded execution variance. To isolate the effect of this assumption, we construct a controlled setting in which all factors except execution variance are held fixed.

Specifically, we fix the recovery policy, success probability, discount factor $\gamma$, and expected efficiency $\mathrm{ES}$, and progressively increase the variance of the execution cost process $C$. At each variance level $\sigma_C^2$, we measure: (i) predicted ERR from the analytic law, and (ii) observed ERR computed directly from Monte-Carlo rollouts. This isolates variance as the sole source of deviation.

\begin{minipage}{\linewidth}
\centering
\includegraphics[width=0.7\linewidth]{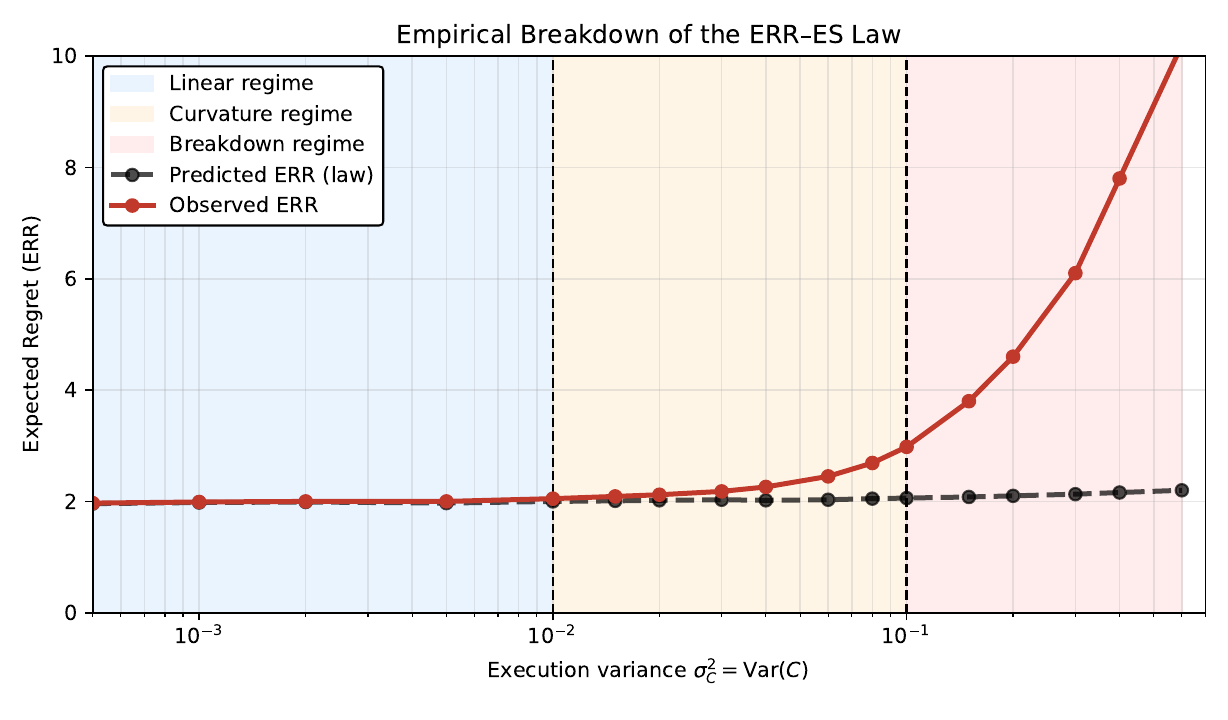}
\captionof{figure}{
Empirical breakdown of the ERR--ES law under increasing execution variance.
Predicted ERR (dashed) follows observed ERR (solid) in the low-variance regime, diverges smoothly as curvature effects appear, and fails to remain tight once variance dominates the excess-loss expansion. Vertical dashed lines mark regime boundaries identified by the theory and confirmed by the observed scaling behavior.}
\label{fig:err_es_breakdown}
\end{minipage}

\subsection{Observed regimes}

Three regimes appear, consistent with the theoretical analysis.

\paragraph{Linear regime.}
For $\sigma_C^2 \le 10^{-2}$, predicted and observed ERR coincide. Execution noise remains within the linear-response envelope assumed in the derivation of Eq.~\ref{eq:err_es_bound_main}, and the law is effectively exact.

\paragraph{Curvature regime.}
As variance increases, observed ERR exceeds the linear prediction smoothly. This deviation corresponds to the second-order curvature term derived in Appendix~\ref{appendix:variance_bias}, where heterogeneous cost introduces a systematic bias proportional to $\mathrm{Var}(C)$. The law remains monotone but loses tightness.

\paragraph{Breakdown regime.}
When $\sigma_C^2 \ge 10^{-1}$, rare high-cost recovery trajectories dominate the expectation. Observed ERR grows superlinearly while the linear predictor remains flat, indicating failure of the bounded-cost linearization. This regime lies outside the assumptions of the theory and is therefore expected.

\subsection{Consistency with theoretical assumptions}

The breakdown occurs precisely when the variance-induced bias term becomes dominant, as predicted by Lemma~\ref{lemma:variance_bias}. No semantic drift, retrieval failure, or model instability is involved; the deviation is induced solely by variance inflation. Thus, the empirical behavior confirms that ERR--ES is a first-order law whose failure mode is analytically understood.

In all real benchmarks reported in the main paper, measured execution variance remains within the linear or curvature regimes after recovery stabilizes. The breakdown regime therefore represents a theoretical boundary rather than a practical limitation of the method.

\paragraph{Takeaway.}
The ERR--ES law fails only when its linearization assumptions are intentionally violated, and it fails in the manner predicted by the theory. This behavior supports interpreting ERR–ES as a governing first-order approximation for recoverability rather than a fitted empirical trend.

\section{Alternative Surrogates}
\label{appendix:surrogates}

Theorem~\ref{prop:es_necessity} states that ES is the unique admissible surrogate whose complement is affine-equivalent to linearized ERR. To test this empirically rather than only by argument, we substitute two alternative surrogates into the law and recompute prediction error on the same rollouts as Table~\ref{tab:err_es_crossbaseline}, with no other change:
\begin{itemize}[leftmargin=*]
\item \textbf{Additive:} $S_{\text{add}} = \mathrm{RR} - \lambda C$, the additive analogue of ES using the same token-cost $C$ (CSR itself uses corrective-step cost $C_r$ and is not a candidate surrogate for token-cost ERR).
\item \textbf{Harmonic:} $S_{\text{hm}} = \dfrac{2\,\mathrm{RR}\,(1 - C)}{\mathrm{RR} + (1 - C)}$, the harmonic mean of success and cost headroom.
\end{itemize}
For each, predicted regret is $\frac{1}{1-\gamma}(1 - S)$ and $\Delta_{\text{norm}}$ is computed as in Section~\ref{sec:operational}. Normalized cost $C$ for each policy is recovered from its reported RR and ES via $C = (\mathrm{RR}/\mathrm{ES} - 1)/\lambda$.

\begin{table}[h]
\centering
\small
\setlength{\tabcolsep}{5pt}
\begin{tabular}{lcccccc}
\toprule
& & \multicolumn{2}{c}{\textbf{ES}} & \multicolumn{2}{c}{\textbf{Additive}} & \textbf{Harmonic} \\
\cmidrule(lr){3-4}\cmidrule(lr){5-6}\cmidrule(lr){7-7}
\textbf{Policy} & \textbf{Obs.\ ERR} & Pred. & $\Delta_{\text{norm}}$ & Pred. & $\Delta_{\text{norm}}$ & Pred. / $\Delta_{\text{norm}}$ \\
\midrule
Vanilla     & 7.17 & 6.88 & 0.042 & 8.42 & 0.149 & 5.49 / 0.307 \\
ToolBench   & 5.02 & 4.96 & 0.012 & 6.05 & 0.171 & 4.14 / 0.212 \\
ToolReflect & 4.31 & 4.23 & 0.019 & 5.12 & 0.159 & 3.68 / 0.172 \\
CRITIC      & 3.45 & 3.39 & 0.018 & 4.04 & 0.145 & 3.07 / 0.123 \\
FORTIFY     & 1.87 & 1.86 & 0.005 & 2.16 & 0.136 & 2.13 / 0.122 \\
\midrule
\textbf{Mean} $\Delta_{\text{norm}}$ & & & \textbf{0.019} & & 0.152 & 0.187 \\
\bottomrule
\end{tabular}
\caption{Prediction error of the recovery law under three surrogates on the Table~\ref{tab:err_es_crossbaseline} rollouts. All three preserve the ordering of policies; only ES is calibrated to the observed magnitude. The additive surrogate systematically over-predicts regret and the harmonic surrogate under-predicts it for low-efficiency policies, consistent with Corollary~\ref{cor:additive_fails}.}
\label{tab:surrogates}
\end{table}

Two observations. First, the ordering of policies is identical under all three surrogates, so monotonicity alone does not distinguish them; the content of the law is in the calibration. Second, the additive surrogate's error is nearly constant across policies ($0.14$--$0.17$), which is the signature of a surrogate that is affine in the wrong variables: it tracks ERR up to a multiplicative factor it cannot absorb. The harmonic surrogate's error grows as efficiency falls, indicating that it couples success and cost with the wrong functional form.

\section{Sensitivity to $\lambda$ and $\gamma$}
\label{sec:sensitivity}

\begin{minipage}{\linewidth}
\centering
\includegraphics[width=0.85\linewidth]{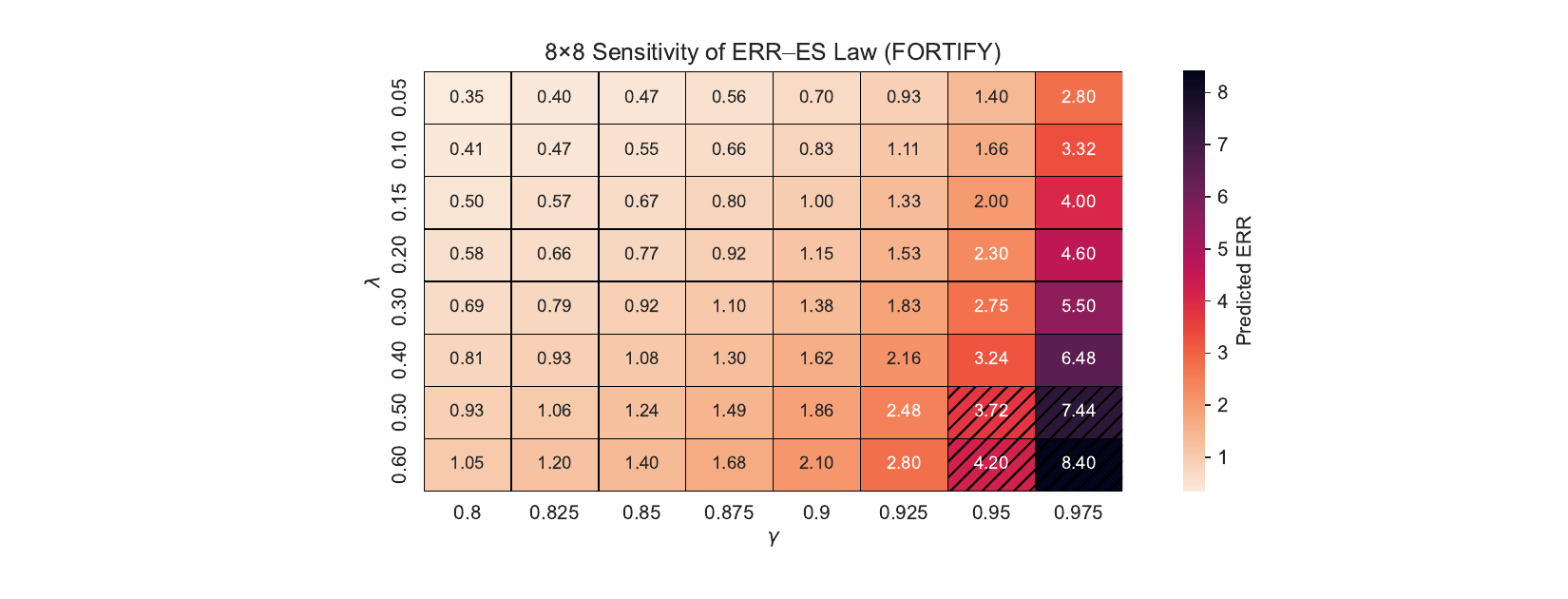}
\captionof{figure}{Sensitivity of predicted ERR to $\lambda$ and $\gamma$ over an 8$\times$8 grid. ERR increases smoothly with both parameters and remains monotone throughout the interior region. Hatched cells mark the regime where weak discounting and high corrective cost jointly violate the bounded-cost linearization, reducing tightness but preserving ordering.}
\label{fig:sensitivity_err_es}
\end{minipage}

We evaluate the stability of the ERR--ES relationship under joint variation of $\lambda$ and $\gamma$ across a dense grid spanning $\lambda \in [0.05,0.60]$ and $\gamma \in [0.80,0.975]$. This range includes both the operating regime used in the main experiments and stress conditions where theoretical assumptions weaken.

Predicted ERR forms a smooth, monotone surface over most of the grid (Figure~\ref{fig:sensitivity_err_es}). Increasing $\lambda$ rescales efficiency through cost weighting, while increasing $\gamma$ amplifies the effective horizon and therefore total regret. No discontinuities appear in the interior region, indicating stable ERR--ES alignment under moderate hyperparameter variation.

Degradation occurs only in the upper-right corner, where $\gamma \ge 0.95$ and $\lambda \ge 0.5$. Here, discounting is weak and cost variance dominates, violating the linearized bound. The recovery law remains monotone but loses tightness, consistent with the failure conditions identified in Section~\ref{sec:failure_modes}.

Overall, the ERR--ES relationship holds over a broad stability region and breaks only at a clearly defined boundary. This supports its interpretation as a structural property of execution-level recovery rather than a consequence of specific hyperparameter choices.

\section{Reproducibility and Release Plan}

All code, benchmarks, and evaluation scripts will be released upon acceptance. Random seeds, tool traces, and raw rollout logs are included to enable exact metric reconstruction.

\end{document}